\documentclass[preprint]{alt2027}
\usepackage{times}

\altauthor{%
 \Name{Dhruv Sarkar} \Email{dhruv.sarkar223@gmail.com}\\
 \addr MIT, Cambridge, USA
 \AND
 \Name{Abhishek Sinha} \Email{abhishek.sinha@tifr.res.in}\\
 \addr Tata Institute of Fundamental Research, Mumbai, India%
}

\usepackage{times}
\usepackage{booktabs}
\usepackage{nicefrac}
\usepackage{microtype}
\usepackage{mathtools}
\usepackage{mathrsfs}
\usepackage{array}
\usepackage{dsfont}
\usepackage{bm}
\usepackage{color}
\usepackage[utf8]{inputenc} 
\usepackage[T1]{fontenc}    
\usepackage{hyperref}       
\usepackage{url}            
\usepackage{booktabs}       
\usepackage{amsfonts}       
\usepackage{nicefrac}       
\usepackage{microtype}      
\usepackage{setspace}
\usepackage{dsfont}
\usepackage{amsmath}
\usepackage{amssymb, bm}
\usepackage[dvipsnames,svgnames]{xcolor}
\usepackage{xcolor}
\usepackage{hyperref}
\usepackage{mathrsfs}
\usepackage{mathtools}
\usepackage{algorithm}
\usepackage{algorithmic}
\usepackage{graphicx}

\usepackage{nicefrac}
\usepackage{framed}

\newcommand{\Regret}{\mathsf{Regret}}

\newcommand{\CCB}{$\mathsf{CCB}$}

\hypersetup{
    colorlinks=true,
    linkcolor=blue,
    urlcolor=red,
    citecolor=NavyBlue
}

\newtheorem{assumption}{Assumption}

\newcommand{\CCV}{\mathsf{CCV}}

\newcounter{commentcounter}

\setcitestyle{authoryear,open={(},close={)}}

\mathcode`*="8000
{\catcode`*=\active\gdef*{\star}}

\newcommand{\SqCB}{$\mathsf{SquareCB}$}
\usepackage[T1]{fontenc}
\usepackage[utf8]{inputenc}
\usepackage{microtype}
\usepackage{mathtools}
\usepackage{booktabs}
\usepackage{bm}
\title[Constrained Contextual Bandits with Adversarial Contexts]{Constrained Contextual Bandits with Adversarial Contexts}
\let\leq\leqslant
\let\geq\geqslant
\let\le\leqslant
\let\ge\geqslant

\begin{document}
\maketitle
\begin{abstract}
We study constrained contextual bandits with adversarial contexts, where each action yields a random reward and incurs a random cost. The conditional mean reward and cost belong to known function classes and are learned through online regression oracles. The benchmark is a stationary randomized policy whose expected cost is non-positive at every observed context. We extend the seminal $\mathsf{SquareCB}$ algorithm of \citet{foster2020beyond}, which reduces the contextual bandits to a regression problem, to the constrained setting. A surrogate reward function is constructed which penalizes the positive part of the expected cost of an action distribution, and an auxiliary variable accumulates its predicted value. Then we use an exponential Lyapunov function to control both regret and cumulative constraint violation via a single inequality. If the cumulative squared prediction error is at most $U_T$, our algorithm achieves $O(\sqrt{KTU_T})$ regret and $O(\sqrt{KTU_T}\log T)$ cumulative expected cost, with high probability, improving upon the previous $O(T^{3/4})$ type bounds obtained in stochastic contextual settings with an unknown context distribution. Our results require neither additional margin conditions nor probabilistic assumptions on the context sequence. 
\end{abstract}

\section{Introduction}
\label{sec:intro}

Contextual bandits provide a fundamental framework for sequential decision-making with side information, with applications ranging from recommendation systems to clinical trials \citep{beygelzimer2010optimal,agarwal2014taming}. They can be viewed as a natural generalization of the classical stochastic Multi-Armed Bandit (MAB) problem where some additional side information conveying an implicit actions-to-rewards mapping is made available to the learner in the form of \emph{contexts}. 
In many practical settings, however, decisions must satisfy additional long-term constraints, such as budget, safety, or fairness requirements. This has led to an extensive literature on \emph{Constrained Contextual Bandits} (\CCB), where the learner must simultaneously maximize reward and control cumulative constraint violation \citep{badanidiyuru2014resourceful, agrawal2014bandits, castiglioni2022online}.

Most existing work on \CCB~relies critically on stochastic assumptions on the context sequence, enabling techniques based on concentration arguments \citep{guo2024stochastic, guo2025stochastic, han2023optimal}. These approaches naturally break down in non-stationary or adversarial environments, where contexts may evolve arbitrarily or depend on the learner’s past actions \citep{harris2024regret,hu2025learning}. Extending \CCB~to adversarial contexts therefore remains a central challenge requiring new design principles and analytical techniques \citep[Remark 2.2]{slivkins2023contextual}. 

In this paper, we obtain improved bounds while going beyond the stochastic context assumption. We develop a reduction scheme that converts the constrained problem into an unconstrained contextual bandit with \emph{adaptively defined} surrogate rewards. Recall that the \SqCB~algorithm, proposed by \citet{foster2020beyond}, reduces unconstrained contextual bandits to online regression, thereby avoiding the need to construct confidence intervals, which can be difficult for nonparametric function classes. Our algorithm builds on the \SqCB~framework and incorporates constraint handling through a Lyapunov-based construction, resulting in an oracle-efficient algorithm that naturally balances exploration, exploitation, and constraint satisfaction. Our key technical contribution is a general regret decomposition inequality (Proposition \ref{prop:decomposition}) that simultaneously controls regret and cumulative constraint violation ($\CCV$). This \emph{single} inequality cleanly separates the roles of exploration (via inverse gap weighting), estimation (via regression oracles), and constraint management (via Lyapunov dynamics), and serves as the foundation for all our results (see Theorem \ref{thm:main}).

Leveraging this framework, we obtain an efficient policy yielding sharp bounds for both regret and $\CCV$~for adversarial contexts. These results improve over prior  guarantees that rely on stochastic contexts, and apply more broadly without requiring assumptions such as known context distributions \citep{guo2025stochastic} or strict feasibility \citep{guo2024stochastic}.  
See Table \ref{tab:related-work} for a summary of key improvements over the state-of-the-art. In summary, we make the following contributions in this work:

\paragraph{Contributions.}
We derive a regret decomposition inequality that separates exploration, reward estimation, and cost estimation. An exponential Lyapunov function then yields $O(\sqrt{KTU_T})$ regret expressed in conditional means and $O(\sqrt{KTU_T}\log T)$ cumulative positive expected cost, where $K$ is the number of actions and $U_T$ bounds the cumulative squared prediction errors. The guarantees hold for adversarially chosen adaptive contexts without any strictly positive margin (Slater) condition. 
\subsection{Related Work}
\paragraph{Regression reductions.}
\citet{badanidiyuru2014resourceful} introduced contextual bandits with resource constraints. For unconstrained contextual bandits, \citet{foster2020beyond} established a reduction to online square-loss regression that applies to adversarial contexts. Their analysis measures prediction error under the same action distribution that generates the observations. We retain this property in the constrained problem. The additional difficulty is that the penalty multiplying the cost prediction changes over time, and its size depends on previous predictions and decisions. Our decision inequality, given in Lemma \ref{lem:decision}, controls these errors through the ordinary squared-loss guarantee, allowing prediction errors of either sign.

\begin{table}[t]
\centering
\caption{\small{Regret and constraint-violation bounds for constrained contextual bandits.
Benchmark is feasible in expectation at each
observed context}.}
\label{tab:related-work}
\begingroup
\footnotesize
\setlength{\tabcolsep}{4pt}
\renewcommand{\arraystretch}{1.17}

\begin{tabular}{@{}
>{\raggedright\arraybackslash}p{0.22\linewidth}
>{\raggedright\arraybackslash}p{0.23\linewidth}
>{\raggedright\arraybackslash}p{0.23\linewidth}
>{\raggedright\arraybackslash}p{0.26\linewidth}
@{}}

\toprule
\textbf{Work} &
\textbf{Regret} &
\textbf{Violation} &
\textbf{Context / condition} \\
\midrule

\citet{slivkins2023contextual}
& $\widetilde O(T^{3/4}U_T^{1/4})$
& $\widetilde O(T^{3/4}U_T^{1/4})$
& Stochastic; margin for $\sqrt T$ rates \\

\addlinespace[2pt]

\citet{guo2024stochastic}
& $\widetilde O(T^{3/4}U_T^{1/4})$
& $\widetilde O(T^{3/4}U_T^{1/4})$
& Unknown i.i.d.\ contexts \\

\addlinespace[2pt]

\citet{guo2025triple}
& $\widetilde O(\sqrt T)$
& $\widetilde O(\sqrt T)$
& \textbf{Known i.i.d. context law}; optimistic estimates \\

\addlinespace[2pt]

\citet{bernasconi2025noregret}
& $\widetilde O(\sqrt{T})$ ($\alpha$-regret)
& $\widetilde O(\sqrt{T})$
& Adversarial; margin $\rho>0$, $\alpha=\frac{\rho}{1+\rho}$ \\

\midrule

\textbf{This work}
& $O(\sqrt{KTU_T})$
& $O(\sqrt{KTU_T}\log T)$
& \textbf{Adversarial contexts}; no known law or margin assumption \\

\bottomrule
\end{tabular}
\endgroup
\end{table}

\paragraph{Constrained contextual bandits.}
The paper by \citet{slivkins2023contextual} gives a modular reduction for stochastic contextual bandits with packing and covering constraints. Their general bounds have order $\widetilde O(T^{3/4}U_T^{1/4})$ for fixed action and resource counts, with square-root rates under an additional margin condition at an optimal solution. They also consider environments whose context distributions change across segments. \citet{guo2024stochastic} combine regression, inverse gap weighting, and quadratic Lyapunov function for unknown i.i.d. contexts. Their queue accumulates observed signed costs with reflection at zero, and their analysis uses drift over several rounds. They obtain $O(T^{3/4})$ Regret and $O(T^{3/4})$ $\CCV$ bounds in the stochastic context setting without any margin assumptions. 

The $O(\sqrt{T})$ $\Regret/\CCV$ guarantee of \citet{guo2025triple} uses the known context law, which is generally unknown in practice, in the algorithm itself. Their dual update in Section~4 averages predicted costs over that law. Their analysis compares successive predictions at the same context before averaging. As their Remark~5.8 explains, using consecutive observed contexts introduces variation between costs at different contexts. Such variation can be of constant order per round even with exact predictions. Knowledge of the context law therefore removes a central difficulty in their constraint analysis. They reserved the question of removing the distributional assumption of the contexts as an important open problem. Their Assumption~3.2 also requires upper confidence estimates for rewards and lower confidence estimates for costs and constraints at \emph{each round}. On the other hand, our square-root rates are obtained through ordinary regression oracle while allowing the context sequence to change arbitrarily.

\citet{bernasconi2025noregret} study bandits with general constraints through adaptive regret minimization and derive guarantees for stochastic and adversarial inputs under their respective benchmarks. Their contextual extension obtains a competitive reward guarantee under a positive feasibility margin. For contextual bandits with knapsacks, \citet{han2023optimal} and \citet{guo2025stochastic} study unknown stochastic contexts with a total budget and stopping when a resource is exhausted. These problems impose a different resource requirement from the continuing interaction considered here.

\paragraph{Constrained online optimization.}
\citet{sinha2024optimal} proposed an online gradient descent based policy to control regret and cumulative constraint violation in constrained online convex optimization (COCO). Their exponential Lyapunov argument motivates our analysis. In contextual bandits, the learner estimates the cost of the action distribution from feedback for the selected action. The decision inequality in Lemma~\ref{lem:decision} accounts for this estimation error and connects the Lyapunov argument to the regression guarantee.

\section{Problem Formulation}
\label{sec:model}

We consider a constrained contextual bandit problem over a horizon of $T$ rounds. At the beginning of round $t$, the learner observes a context $x_t$ from the set of all possible contexts   $\mathcal X$ (we allow $|\mathcal{X}|$ to be infinitely large). The context may depend on the complete history before round $t$ and could be chosen by an adaptive adversary. Upon observing $x_t$, the learner selects an action distribution $\pi_t(\cdot\mid x_t)$ over $K$ possible actions (also known as \emph{arms}) and randomly samples an action $a_t\sim\pi_t(\cdot\mid x_t)$ from the selected distribution. It then observes a random reward $f_t(x_t,a_t)\in[-1,1]$ and a random cost $g_t(x_t,a_t)\in[-1,1]$. The realized cost could be either positive and negative. 
We make the following realizability assumption. 

\begin{assumption}[Realizability]
\label{assum:realizability}
For known function classes $\mathcal F,\mathcal G\subseteq[-1,1]^{\mathcal X\times[K]}$, there exist fixed functions $f^\star\in\mathcal F$ and $g^\star\in\mathcal G$ such that
\begin{equation}
\begin{split}
\mathbb E[f_t(x_t,a_t)\mid x_t, a_t]&=f^\star(x_t,a_t),\\
\mathbb E[g_t(x_t,a_t)\mid x_t,a_t]&=g^\star(x_t,a_t).
\end{split}
\label{eq:realizability}
\end{equation}
\end{assumption}
In other words, we assume that the reward and cost observations are generated from fixed conditional distributions given the current context and action, with conditionally independent noise across rounds.
The ground truths $f^\star$ and $g^\star$ are unknown and must be learned from past observations. Assumption~\ref{assum:realizability} allows dependence between reward and cost, and allows the contexts to be chosen adaptively.

\subsection{Offline Benchmark}
Let $\Pi$ be the set of stationary randomized policies $\pi:\mathcal X\to\Delta_K$, where $\Delta_K$ is the probability simplex on $[K]$. Since we make no distributional assumption on the context sequence and allow contexts to be chosen adaptively by an adversary, we adopt a benchmark that is feasible in expectation at each observed context, as defined next.

\begin{definition}[Feasibility in expectation]
\label{def:feasible}
A stationary policy $\pi$ is feasible for the observed context sequence if
\begin{equation}
\mathbb E_{a\sim\pi(\cdot\mid x_t)}g^\star(x_t,a)
\equiv \langle g^\star(x_t),\pi(\cdot\mid x_t)\rangle\le0,
\qquad \forall 1\le t\le T,
\label{eq:feasible}
\end{equation}
\end{definition}
where for any function $h$ on context--action pairs, we write $h(x)$ to denote the vector $(h(x,a):a\in[K])$.

We assume that a feasible policy exists. The goal of the learner is to perform close to an \emph{optimal-in-hindsight} stationary policy that maximizes the expected cumulative rewards  while satisfying the cost constraints in the long term. Formally, for the realized context sequence, the benchmark is given by 
\begin{equation}
\pi^\star\in\arg\max_{\pi\in\Pi}
\left\{\sum_{t=1}^T\langle f^\star(x_t),\pi(\cdot\mid x_t)\rangle:
\pi\text{ satisfies \eqref{eq:feasible}}\right\}.
\label{eq:benchmark}
\end{equation}

\subsection{Performance Metrics}
We measure performance of a policy using regret and cumulative constraint violation. For any feasible benchmark $\pi^\star$, define
\begin{align}
\mathsf{Regret}_t(\pi^\star)
&=\sum_{\tau=1}^t\langle f^\star(x_\tau),\pi^\star(\cdot\mid x_\tau)\rangle
-\sum_{\tau=1}^t f_\tau(x_\tau,a_\tau),\label{eq:regret}\\
\mathsf{CCV}_t&=\sum_{\tau=1}^t g_\tau(x_\tau,a_\tau).\label{eq:ccv}
\end{align}
Positive values of $\mathsf{CCV}_t$ measure excess total cost up to round $t$. For the analysis, we also use regret expressed in terms of conditional means,
\begin{equation}
\overline{\mathsf{Regret}}_t(\pi^\star)
=\sum_{\tau=1}^t\langle f^\star(x_\tau),
\pi^\star(\cdot\mid x_\tau)-\pi_\tau(\cdot\mid x_\tau)\rangle.
\label{eq:meanregret}
\end{equation}
Both quantities can be negative because the learner may violate the constraint satisfied by the benchmark. 
In our analysis, due to its monotonicity property, it would be easier to control the stronger violation quantity instead
%
\begin{equation}
\mathsf{CCV}_t^+
=\sum_{\tau=1}^t
\big(\langle g^\star(x_\tau),\pi_\tau(\cdot\mid x_\tau)\rangle\big)^+,
\qquad \textrm{where} ~(u)^+=\max\{0,u\}.
\label{eq:positiveccv}
\end{equation}
It is easy to see that,
\[
\left(\sum_{\tau=1}^t\langle g^\star(x_\tau),\pi_\tau(\cdot\mid x_\tau)\rangle\right)^+
\le\mathsf{CCV}_t^+.
\]
A standard martingale argument converts control of $\mathsf{CCV}_t^+$ into control of the realized total cost $\mathsf{CCV}_t$ defined in Eqn.\ \eqref{eq:ccv}.

\subsection{Online Regression Oracle}
\label{sec:oracle}
The learner interacts with $\mathcal F$ and $\mathcal G$ through online regression oracles, denoted collectively by $\mathcal O_{\mathrm{sq}}$. Before drawing $a_t$, the oracles return predictions $\hat f_t(x_t,a),\hat g_t(x_t,a)\in[-1,1]$ for every action $a \in [K]$. After the action $a_t$ is played, the oracles are then updated using only the observed reward $f_t(x_t, a_t)$ and cost $g_t(x_t, a_t)$. Predictions may be clipped to $[-1,1]$ without increasing squared error.

We use the standard regression guarantee in the following form. For a given failure probability $\delta\in(0,1)$, a known bound $U_T\ge1$ satisfies, jointly with probability at least $1-\delta$,
\begin{equation}
\begin{split}
\sum_{t=1}^T\mathbb E_{a\sim\pi_t(\cdot\mid x_t)}
\big(\hat f_t(x_t,a)-f^\star(x_t,a)\big)^2&\le U_T,\\
\sum_{t=1}^T\mathbb E_{a\sim\pi_t(\cdot\mid x_t)}
\big(\hat g_t(x_t,a)-g^\star(x_t,a)\big)^2&\le U_T.
\end{split}
\label{eq:oracle}
\end{equation}
The expectations are over the displayed action distribution, conditional on the observed contexts. We suppress the dependence of $U_T$ on $\delta$ when it is clear from context.

Under realizability, a bound on online squared-loss regret implies \eqref{eq:oracle} \citep{foster2023foundations}. For finite classes, the usual exponential weights algorithm gives
\[
U_T=O\!\left(1+\log\max\{|\mathcal F|,|\mathcal G|\}+\log(1/\delta)\right).
\]
For $d$-dimensional linear classes with bounded feature and parameter norms, standard online regression gives $U_T=O(1+d\log(1+T)+\log(1/\delta))$; see \citet{foster2023foundations,orabona2019modern}. Our proposed contextual bandit algorithm uses the regression guarantee only through $U_T$ and is otherwise oblivious to the actual regression algorithm used.

\section{Preliminaries}
\label{sec:algorithm}

\begin{algorithm2e}[t]
\caption{\SqCB: Unconstrained Contextual Bandits with Regression Oracles}
\label{ccb}
\KwIn{Online regression oracle $\mathcal{O}_{\mathrm{sq}}$ and parameter $\gamma>0$}

\For{$t=1,\ldots,T$}{
    Receive context $x_t$\;
    
    Ask $\mathcal{O}_{\mathrm{sq}}$ to predict the loss of each action,
    obtaining $\widehat{v}_t(x_t, 1),\ldots,\widehat{v}_t(x_t, K)$\;
    
    Compute $p_t\in\Delta(K)$ according to Eqn.~\eqref{igw-opt-form}\;
    
    Sample $a_t\sim p_t$, observe the loss $v_t(x_t, a_t)$, and feed
    $(x_t,a_t,v_t(x_t, a_t))$ to $\mathcal{O}_{\mathrm{sq}}$\;
}
\end{algorithm2e}

The $\mathsf{SquareCB}$ framework utilizes an online regression oracle $\mathcal{O}_{\textrm{sq}}$ to select an exploratory action distribution, called Inverse Gap Weighting ($\mathsf{IGW}$). See Algorithm \ref{ccb} for the pseudocode. Since the context does not play a role in this lemma (all results hold round-wise), we drop the contexts from our notation and focus solely on the actions.  We first recall the main \SqCB~inequality in Lemma \ref{lem:igw} and then describe the extension needed for the constrained setting.

\subsection{Inverse Gap Weighting ($\mathsf{IGW}$)}
Given predicted losses $\hat v\in\mathbb R^K$ and $\gamma>0$, inverse gap weighting ($\mathsf{IGW}$) selects a distribution $p=\mathsf{IGW}_\gamma(\hat v)$ with
\begin{equation}\label{eq:igw}
p(a)=\frac{1}{\lambda_0+2\gamma(\hat v(a)-\min_b\hat v(b))},
\qquad a\in[K],
\end{equation}
where $\lambda_0\in[1,K]$ normalizes the probabilities. Equivalently, $p$ minimizes the convex function
\begin{eqnarray} \label{igw-opt-form}
\langle\hat v,p\rangle-\frac{1}{2\gamma}\sum_{a=1}^K\log p(a)
\end{eqnarray}
over the probability simplex $\Delta_K$. As we will see later, the alternative optimization form \eqref{igw-opt-form} can be readily generalized to non-linear reward setting. The following inequality underlying the regression reduction is well-known in the literature \citep{foster2020beyond,foster2023foundations}. 

\begin{lemma}[$\mathsf{IGW}$ inequality \citep{foster2023foundations}]
\label{lem:igw}
Let $p=\mathsf{IGW}_\gamma(\hat v)$. For every $v\in\mathbb R^K$ and $\mu\in\Delta_K$,
\begin{equation}
\langle v,p-\mu\rangle\le\frac{K}{2\gamma}
+\gamma\mathbb E_{a\sim p}(v(a)-\hat v(a))^2.
\label{eq:igwbound}
\end{equation}
\end{lemma}

The LHS can be thought as a generic term in the regret expression. The first term on the right controls exploration, and the second measures estimation error under the selected distribution $p$. In our setting, we will need to use a generalization of the above result to the Hinge loss setting (cf. Lemma \ref{lem:decision} below). Since the proof of Lemma \ref{lem:igw} can be obtained as a special case of Lemma \ref{lem:decision} by setting $\beta=0,$ we do not include it here.


\subsection{Generalization of the $\mathsf{IGW}$ inequality to Hinge Loss}
In the following, we extend the above $\mathsf{IGW}$ inequality to the hinge loss setting, which will be central to our development. Compared to Lemma \ref{lem:igw}, which bounds the regret for a linear function, Lemma \ref{lem:decision} bounds the regret for a linear plus hinge loss function (cf. the LHS of Eqn. \eqref{eq:igwbound} and Eqn.\ \eqref{eq:decisionbound}). The condition $\langle g, \mu \rangle \leq 0$ below corresponds to any benchmark which is feasible in expectation \eqref{eq:feasible}.
\begin{lemma}
\label{lem:decision}
Fix the vectors $\hat f,\hat g,f,g\in[-1,1]^K$, and the scalars $\beta\ge0$, and $\gamma>0$. Let the distribution $p \in \Delta_K$ maximize the concave function
\begin{eqnarray} \label{opt}
\langle\hat f,p\rangle-\beta(\langle\hat g,p\rangle)^+
+\frac{1}{2\gamma}\sum_{a=1}^K\log p(a). 
\end{eqnarray}
 For every $\mu\in\Delta_K$ with $\langle g,\mu\rangle\le0$, define the expected squared errors
$e_f=\mathbb E_{a\sim p}(f(a)-\hat f(a))^2$ and
$e_g=\mathbb E_{a\sim p}(g(a)-\hat g(a))^2$. Then
\begin{equation}
\langle f,\mu-p\rangle+\beta(\langle\hat g,p\rangle)^+
\le\frac{K}{2\gamma}+2\gamma e_f+2\gamma\beta^2e_g+\beta\sqrt{e_g}.
\label{eq:decisionbound}
\end{equation}
\end{lemma}

\begin{proof}
Since $p$ is the optimal probability distribution maximizing the objective function \eqref{opt}, the first-order optimality conditions (KKT)  give a scalar
$s\in\partial(u)^+|_{u=\langle\hat g,p\rangle}$ and a constant $\nu$ such that
\begin{equation}
\hat f(a)-\beta s\hat g(a)+\frac{1}{2\gamma p(a)}=\nu,
\qquad
s\langle\hat g,p\rangle=(\langle\hat g,p\rangle)^+.
\label{eq:kkt}
\end{equation}
In particular, $0\le s\le1$. Define $d=f-\hat f$, $e=\hat g-g$, and $v=d+\beta s e$. Multiplying \eqref{eq:kkt} by $\mu(a)-p(a)$ and summing gives after some simple algebra
\begin{align}
\langle f,\mu-p\rangle+\beta(\langle\hat g,p\rangle)^+
={}&\frac{K}{2\gamma}-\frac{1}{2\gamma}\sum_a\frac{\mu(a)}{p(a)}
+\langle v,\mu-p\rangle\nonumber\\
&+\beta s\langle e,p\rangle+\beta s\langle g,\mu\rangle.
\label{eq:kktidentity}
\end{align}
The last term is non-positive by feasibility. Finally, we can bound 
\begin{eqnarray}\label{gen-igw}
\langle v,\mu-p\rangle &=& \sum_a v(a) \sqrt{p(a)}\bigg(\frac{\mu(a)}{\sqrt{p(a)}} - \sqrt{p(a)}\bigg) \nonumber\\
&\stackrel{(a)}{\leq} & \gamma\sum_a p(a)v(a)^2
+\frac{1}{4\gamma}\sum_a \bigg(\frac{\mu(a)}{\sqrt{p(a)}} - \sqrt{p(a)}\bigg)^2 \nonumber\\
&=& \gamma\sum_a p(a)v(a)^2
+\frac{1}{4\gamma}\big(\sum_a \frac{\mu(a)^2}{p(a)}-1\big),
\end{eqnarray}
where in step (a), we have used the AM-GM inequality $xy \leq \gamma x^2 + \frac{y^2}{4\gamma}, \gamma >0,$ for each term in the summation and the fact that both $\mu$ and $p$ are probability distributions. 
Since $\mu(a)^2\le\mu(a)$, the last term in Eqn.\ \eqref{gen-igw} is absorbed by the second term in Eqn.\ \eqref{eq:kktidentity}, which, together, is upper bounded by zero. Also,
\begin{eqnarray*}
\sum_a p(a)v(a)^2\le 2\sum_a p(a)\big(d(a)^2 + \beta^2 s^2 e(a)^2\big) \leq  2e_f+2\beta^2 e_g,
\end{eqnarray*}
where we have used the fact that $(x+y)^2 \leq 2(x^2+y^2)$ and $0 \leq s \leq 1.$
Furthermore,
\begin{eqnarray*}
	\beta s\langle e,p\rangle = \beta \sum_a \big(se(a) \sqrt{p(a)}\big) \sqrt{p(a)} \stackrel{(\textrm{Cauchy-Schwarz})}{\leq}\beta\sqrt{e_g}.
\end{eqnarray*}
Substituting the above bounds into Eqn.\ \eqref{eq:kktidentity} proves \eqref{eq:decisionbound}.
\end{proof}

\section{Algorithm Design and Analysis}
\label{sec:analysis}
Unlike unconstrained contextual bandits, constrained contextual bandits require simultaneously controlling regret and cumulative constraint violation ($\mathsf{CCV}$). In this section, we present a simple reduction, summarized in Algorithm (\ref{alg:main}), that adapts the \SqCB~framework, originally developed for the unconstrained setting, to accommodate constraints. The key idea is to construct an adaptively weighted surrogate reward \eqref{eq:surrogate} that combines the predicted reward and cost functions, with the weight on the cost term adjusted over time. Figure (\ref{flowchart}) illustrates the resulting algorithmic pipeline.
\begin{figure*}
\centering
	\includegraphics[scale=0.9]{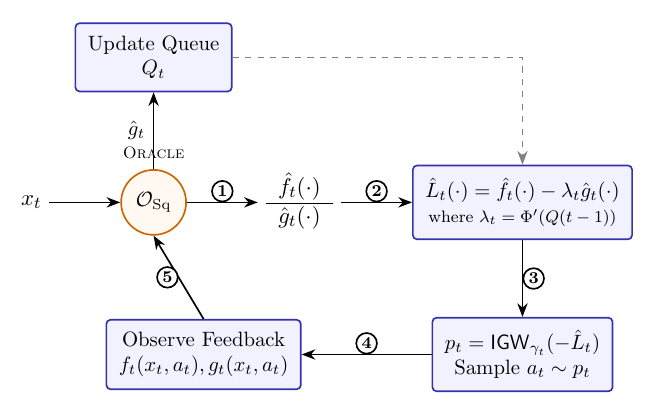}
	\caption{\small{A schematic of the proposed algorithmic scheme for the constrained contextual bandit problem. The numbers within the circles show the sequence of operations performed at any round $t \geq 1$. The variable $Q(t)$ denotes the accumulated positive part of the predicted conditional means \eqref{eq:queue} and $\mathsf{IGW}_{\gamma}(\cdot)$ denotes the generalized inverse gap weighting \eqref{eq:decision} with learning rate $\gamma >0$.} 
 }
	\label{flowchart}
\end{figure*} 
\paragraph{Comparison with \citet{guo2024stochastic}.}
Our approach shares some similarities with the LOE2D framework of \citet{guo2024stochastic}, but differs in two key technical aspects that enable our guarantees under adversarial contexts. First, our virtual-queue update tracks the predicted conditional mean cost, ($\hat{g}$), whereas the update in \citet[Eqn.~(5)]{guo2024stochastic} tracks the realized random cost ($g$). By tracking predicted conditional mean costs rather than noisy realizations, our update avoids the additional variability introduced by randomness of the realized costs.

Second, our exploration parameter ($\gamma_t$) for the ($\mathsf{IGW}$) policy is chosen adaptively using the history of the cumulative positive predicted costs (see \eqref{eq:parameters}), whereas the corresponding parameter $\beta_t$ in \citet[Eqn.~(5)]{guo2024stochastic} depends on the current queue value. This history-dependent choice yields the regret decomposition inequality \eqref{eq:decomposition}, which is central to our analysis under adversarial contexts.

These design choices allow us to move beyond the stochastic-context setting of \citet{guo2024stochastic}. Furthermore, by combining the regret decomposition with an exponential Lyapunov function, we obtain improved ($\mathsf{CCV}$) guarantees relative to the analysis of \citet{guo2024stochastic}, which employs the usual quadratic Lyapunov function.
\subsection{Construction of the Surrogate Rewards}

At each round \(t\), the learner must balance reward maximization against cumulative constraint violation (\(\mathsf{CCV}\)). To this end, we maintain a virtual queue \(Q(t)\) that accumulates the positive part of the predicted conditional mean cost under the learner's action distribution:
\begin{equation}
Q(t)=Q(t-1)+
\big(\langle\hat g_t(x_t),\pi_t(\cdot\mid x_t)\rangle\big)^+,
\label{eq:queue}
\end{equation}
with \(Q(0)=0\). Unlike a queue that tracks realized costs, this update uses the predicted conditional mean cost, thereby avoiding the additional variability due to observing noisy cost realizations. Since \(\hat g_t(x_t,a)\in[-1,1]\), the increment in \eqref{eq:queue} belongs to \([0,1]\) and is known before the action is sampled at round \(t\). Moreover, the positive-part operation ensures that \(Q(t)\) is non-decreasing, a property that will be essential to our analysis. As in \SqCB, the observed cost is subsequently used to update the regression oracle.

To quantify the trade-off between reward and constraint violation, let \(\Phi:\mathbb R\to\mathbb R_+\) be a differentiable, non-decreasing, convex Lyapunov (potential) function. By convexity, the increment in the potential satisfies
\begin{equation}
\begin{aligned}
\Phi(Q(t))-\Phi(Q(t-1))
&\leq \Phi'(Q(t))\big(Q(t)-Q(t-1)\big)\\
&=\Phi'(Q(t))
\big(\langle\hat g_t(x_t),
\pi_t(\cdot\mid x_t)\rangle\big)^+.
\end{aligned}
\label{eq:potential_increment}
\end{equation}
This inequality motivates penalizing positive predicted costs according to the causal quantity \(\Phi'(Q(t-1))\) as $\Phi'(Q(t))$ is not yet available at the beginning of round $t$. Because \(Q(t-1)\) is determined before the current-round predictions and action are used, this choice yields a penalty that is available when selecting the action distribution at round $t$. The resulting penalty increases with the accumulated predicted violation as \(\Phi'(\cdot)\) is non-decreasing. The proof of Theorem \ref{thm:main} will rigorously justify the validity of this approximation by appealing to the smoothness of the chosen Lyapunov function. 

Following the drift-plus-penalty principle of \citet{neely2010stochastic}, we now construct a surrogate reward that trades off the predicted reward against a penalty for positive predicted cost measured in terms of the increment of the Lyapunov potential at round $t$. Specifically, for any action distribution \(\pi\in\Delta_K\), we define the following surrogate reward function
\begin{equation}
\hat L_t(x_t,\pi)
=\langle\hat f_t(x_t),\pi\rangle
-\Phi'(Q(t-1))
\big(\langle\hat g_t(x_t),\pi\rangle\big)^+.
\label{eq:surrogate}
\end{equation}
The coefficient \(\Phi'(Q(t-1))\) adaptively adjusts the cost penalty according to the cumulative positive predicted cost. Thus, the surrogate reward encourages high-reward actions while increasingly penalizing distributions with positive predicted cost as the virtual queue grows.

We incorporate this surrogate reward function into the \SqCB~framework (Algorithm~\ref{ccb}) by replacing its linear predicted reward with the generally non-linear function \(\hat L_t(\cdot)\) and using a time-varying learning-rate sequence \(\{\gamma_t\}_{t\geq1}\). In particular, at round \(t\), the learner selects the following generalized $\mathsf{IGW}$ action distribution:
\begin{equation}
\pi_t(\cdot\mid x_t)
=\arg\max_{\pi\in\Delta_K}
\left\{
\hat L_t(x_t,\pi)
+\frac{1}{2\gamma_t}\sum_{a=1}^K\log\pi(a)
\right\}.
\label{eq:decision}
\end{equation}
The logarithmic regularizer promotes exploration by assigning a barrier penalty to distributions on the boundary of the simplex. Algorithm~\ref{alg:main} summarizes the resulting procedure.

\begin{algorithm2e}[t]
\caption{Constrained Contextual Bandits with Regression Oracles}
\label{alg:main}
\DontPrintSemicolon
\LinesNumbered
\KwIn{Horizon $T$, regression oracles $\mathcal O_{\mathrm{sq}}$, bound $U_T\ge1$, and Lyapunov function $\Phi(\cdot)$.}
Initialize $Q(0)=0$\;
\For{$t=1,\ldots,T$}{
Observe $x_t$ and obtain predictions $\hat f_t(x_t,\cdot)$ and $\hat g_t(x_t,\cdot)$\;
Set $z_t$ and $\gamma_t$ according to Eqn.~\eqref{eq:parameters}\;
Construct $\hat L_t$ as in Eqn.~\eqref{eq:surrogate} and select $\pi_t(\cdot\mid x_t)$ using Eqn.~\eqref{eq:decision}\;
Draw $a_t\sim\pi_t(\cdot\mid x_t)$ and observe $f_t(x_t,a_t)$ and $g_t(x_t,a_t)$\;
Update the two regression oracles using the observed reward and cost\;
Update $Q(t)$ according to Eqn.~\eqref{eq:queue}\;
}
\end{algorithm2e}

\paragraph{Computation.}
The objective in \eqref{eq:decision} is strictly concave over the interior of the simplex and tends to \(-\infty\) as any coordinate approaches zero. Consequently, it admits a unique maximizer with strictly positive coordinates. Each round therefore requires solving a finite-dimensional convex optimization problem over \(K\) action probabilities, which can be handled using standard convex optimization methods. 

In the following section, we analyze Algorithm \ref{alg:main} under a suitable choice of the learning-rate sequence $\big(\gamma_t, t\geq 1\big),$ and an appropriate Lyapunov function $\Phi(\cdot)$. The analysis consists of two steps. We first establish a regret decomposition inequality for the learning rates specified in \eqref{eq:parameters}. Finally, in Theorem \ref{thm:main}, we choose an exponential Lyapunov function to derive our main guarantees. Throughout the analysis, we work on the high-probability event \eqref{eq:oracle}. All deterministic inequalities established on this event hold simultaneously for every prefix \(t\leq T\) and every feasible benchmark.

\subsection{Analysis: The Regret Decomposition Inequality}
Central to our analysis is the regret decomposition, which connects the $\mathsf{\Regret}$ and $\CCV$ within a single inequality \eqref{eq:decomposition}. Our performance bounds are derived by directly solving this inequality with a suitable choice of the Lyapunov function (see Theorem \ref{thm:main} below). 
\begin{proposition}[Regret Decomposition Inequality]
\label{prop:decomposition}
For any fixed convex Lyapunov function $\Phi(\cdot),$ set the learning rates as
\begin{equation}
\gamma_t=\frac{1}{2z_t}\sqrt{\frac{K}{U_T}\sum_{\tau=1}^t z_\tau}, \quad \textrm{where}~~z_t=\max\{1,\Phi'(Q(t-1))^2\}.
\label{eq:parameters}
\end{equation}
On the high-probability event \eqref{eq:oracle}, Algorithm~\ref{alg:main} satisfies, for every $t\le T$ and every feasible benchmark $\pi^*$,
\begin{equation}
\overline{\mathsf{Regret}}_t(\pi^\star)
+\sum_{\tau=1}^t\Phi'(Q(\tau-1))\big(Q(\tau)-Q(\tau-1)\big)
\le5\sqrt{KU_T\sum_{\tau=1}^t z_\tau}.
\label{eq:decomposition}
\end{equation}
\end{proposition}

\begin{proof}
For convenience of writing, define $\beta_\tau \equiv \Phi'(Q(\tau-1))$, and 
$S_\tau \equiv \sum_{s=1}^\tau z_s$. Let $e_{\tau, f}\equiv \mathbb E_{a\sim \pi_\tau(\cdot \mid x_\tau)}(f(a, x_\tau)-\hat f(a, x_\tau))^2$ and
$e_{\tau, g}\equiv \mathbb E_{a\sim \pi_\tau(\cdot \mid x_\tau)}(g(a, x_\tau)-\hat g(a, x_\tau))^2$ be the two expected squared errors on round $\tau$ in \eqref{eq:oracle} under the policy $\pi_\tau$. We now apply Lemma~\ref{lem:decision} with the true and predicted vectors for the context $x_\tau$, the learner's distribution $\pi_\tau(\cdot \mid x_\tau)$, and the stationary feasible benchmark $\mu=\pi^\star(\cdot\mid x_\tau)$. 
Lemma \ref{lem:decision}, in this setting, yields 
\begin{eqnarray}
&&\langle f(x_\tau),\pi^*(\cdot|x_\tau)-\pi_\tau(\cdot|x_\tau) \rangle+\beta_\tau(\langle\hat g_\tau(x_\tau),\pi_\tau(\cdot|x_\tau)\rangle)^+ \nonumber\\
&&\le\frac{K}{2\gamma_\tau}+2\gamma_\tau e_{\tau, f}+2\gamma_\tau\beta_\tau^2e_{\tau, g}+\beta_\tau\sqrt{e_{\tau,g}}.
\label{eq:decisionbound2}
\end{eqnarray}

We now sum up the resulting inequalities for $1\leq \tau \leq t$.
The first term can be upper bounded as
\[
\sum_{\tau=1}^t\frac{K}{2\gamma_\tau}
=\sqrt{KU_T}\sum_{\tau=1}^t\frac{z_\tau}{\sqrt{S_\tau}}
\le2\sqrt{KU_TS_t}.
\]
The last inequality follows by integrating $x^{-1/2}$ over each interval $[S_{\tau-1},S_\tau]$. Furthermore, the learning rate $\gamma_\tau$ for any round $1\leq \tau \leq t$ can be upper bounded as 
\begin{eqnarray*}
	\gamma_\tau = \frac{1}{2 \max(1, \beta_\tau^2)} \sqrt{\frac{K}{U_T}S_\tau} \leq   \frac{1}{2 \max(1, \beta_\tau^2)} \sqrt{\frac{K}{U_T}S_t}.
\end{eqnarray*}
Hence, it follows that for any $1\leq\tau \leq t,$ we have: 
\begin{eqnarray*}
	2 \gamma_\tau \leq \sqrt{\frac{KS_t}{U_T}}, \qquad ~~ \textrm{and}~~~~~ 2 \beta_\tau^2 \gamma_\tau \leq \sqrt{\frac{KS_t}{U_T}}.
\end{eqnarray*}
Thus the two squared-error terms in Eqn.\ \eqref{eq:decisionbound2} can be upper bounded as
\begin{align*}
\sum_{\tau=1}^t2\gamma_\tau e_{\tau,f}
&\le\sqrt{KS_t/U_T}\sum_{\tau=1}^t e_{\tau,f}
\le\sqrt{KU_TS_t},\\
\sum_{\tau=1}^t2\gamma_\tau\beta_\tau^2 e_{\tau,g}
&\le\sqrt{KS_t/U_T}\sum_{\tau=1}^t e_{\tau,g}
\le\sqrt{KU_TS_t},
\end{align*}
where we have used the high-probability error bound of the oracle from Eqn.\ \eqref{eq:oracle}. 
Finally, Cauchy--Schwarz gives
\[
\sum_{\tau=1}^t\beta_\tau\sqrt{e_{\tau,g}}
\le\sqrt{\sum_{\tau=1}^t\beta_\tau^2}
\sqrt{\sum_{\tau=1}^t e_{\tau,g}}
\le\sqrt{U_TS_t}\le\sqrt{KU_TS_t}.
\]
Summing up \eqref{eq:decisionbound2} and substituting these bounds prove the result. 
\end{proof}

\subsection{Main Result}
\label{sec:mainresult}
We finally solve the regret decomposition inequality \eqref{eq:decomposition} with the following choice of the Lyapunov function
\begin{equation}
\Phi(q)=e^{\lambda q},
\qquad~~ \textrm{where}~~
\lambda=\frac{1}{20\sqrt{KTU_T}}.
\label{eq:potential}
\end{equation}
The following is our main result. 
\begin{theorem}
\label{thm:main}
Assume realizability (Assumption \ref{assum:realizability}) and choose $U_T\ge1$ so that the bound in Eqn.\ \eqref{eq:oracle} holds with failure probability $\delta/3$. Run Algorithm~\ref{alg:main} with the exponential Lyapunov function given in Eqn.\ \eqref{eq:potential}. Then with probability at least $1-\delta$, the following Regret and CCV bounds hold simultaneously for every $t\le T$ and every feasible benchmark:
\begin{align}
\overline{\mathsf{Regret}}_t(\pi^\star)
&\le5\sqrt{KU_Tt}+1,\label{eq:mainregret}\\
\mathsf{CCV}_t^+
&\le20\sqrt{KTU_T}\log\!\left(4t+10\sqrt{KU_Tt}+2\right)
+\sqrt{tU_T},\label{eq:mainccv}\\
\mathsf{Regret}_t(\pi^\star)
&\le5\sqrt{KU_Tt}+1+\sqrt{8t\log(3T/\delta)},\label{eq:realizedregret}\\
\mathsf{CCV}_t
&\le\mathsf{CCV}_t^++\sqrt{8t\log(3T/\delta)}.
\label{eq:realizedccv}
\end{align}
\end{theorem}

In particular, regret expressed in conditional means has order $O(\sqrt{KTU_T})$, and the sum of positive expected costs has order $O(\sqrt{KTU_T}\log(T+1))$. Realized Regret and CCV has the additional $O(\sqrt{t})$ concentration terms.

\begin{proof}
  Since $Q(\tau)$ is non-decreasing and $\Phi'(Q(t)) = \lambda \Phi(Q(t)),$ we have 
\[
\sum_{\tau=1}^t z_\tau
\le t\max\{1,\lambda^2 \Phi(Q(t))^2\} \leq t (1+\lambda^2 \Phi(Q(t))^2).
\]
This is a crucial conceptual step as it allows us to upper-bound the history-dependent sum $\sum_{\tau=1}^t z_\tau,$ which involves $t$ variables $\{Q(\tau), 1\leq \tau \leq t\},$ with a single terminal quantity involving only $Q(t).$ Define $A_t \equiv 5\sqrt{KU_Tt}$.
Hence, it follows from Proposition~\ref{prop:decomposition} that  for any feasible benchmark $\pi^*,$ we have
\begin{equation}
\overline{\mathsf{Regret}}_t(\pi^*)
+\sum_{\tau=1}^t\Phi'(Q(\tau-1))\big(Q(\tau)-Q(\tau-1)\big)
\le A_t(1+\lambda \Phi(Q(t))).
\label{eq:absorb1}
\end{equation}
We now simplify the second term of the LHS of \eqref{eq:absorb1}.
Since $\hat{g} \in [-1,1],$ each queue increment $u=Q(\tau)-Q(\tau-1)$ belongs to $[0,1]$. The inequality $e^{\lambda u}-1\le e^\lambda\lambda u,$ which follows from convexity of the map $u \mapsto e^{\lambda u},$ yields
\begin{align*}
\sum_{\tau=1}^t\Phi'(Q(\tau-1))\big(Q(\tau)-Q(\tau-1)\big) &= \sum_{\tau=1}^t e^{\lambda Q(\tau-1)} \lambda (Q(\tau) - Q(\tau-1))\\
&\ge e^{-\lambda}\sum_{\tau=1}^t
\big(e^{\lambda Q(\tau)}-e^{\lambda Q(\tau-1)}\big)\\
&=e^{-\lambda}(\Phi(Q(t))-1).
\end{align*}
Substituting the above bound into \eqref{eq:absorb1}, and using $A_t\lambda\le\nicefrac14$ and $\lambda\le\nicefrac{1}{20}$, gives
\begin{equation*}
\overline{\mathsf{Regret}}_t(\pi^*)+\tfrac34 (\Phi(Q(t))-1 ) \le A_t+\tfrac{1}{4}\Phi(Q(t)).
\label{eq:absorb3}
\end{equation*}
The above, after transposition, yields the following relation which directly connects the $\Regret$ and $\CCV:$ 
\begin{equation}
\overline{\mathsf{Regret}}_t(\pi^*)+\tfrac12 \Phi(Q(t))\le A_t+1.
\label{eq:absorb2}
\end{equation}
Lower bounding the nonnegative exponential term in \eqref{eq:absorb2} by zero proves \eqref{eq:mainregret}. Since rewards lie in $[-1,1]$, we also have $\overline{\mathsf{Regret}}_t(\pi^*)\ge-2t$. Hence, from \eqref{eq:absorb2}, we also have the following bound on $Q(t):$
\begin{equation}
Q(t)\le20\sqrt{KTU_T}\log\!\left(4t+10\sqrt{KU_Tt}+2\right).
\label{eq:qbound}
\end{equation}

The positive-part function is $1$-Lipschitz. Applying this fact to the true and predicted mean costs, then using Cauchy--Schwarz and \eqref{eq:oracle}, gives
\begin{align}
\mathsf{CCV}_t^+
&\le Q(t)+\sum_{\tau=1}^t
\left|\langle g^\star(x_\tau)-\hat g_\tau(x_\tau),
\pi_\tau(\cdot\mid x_\tau)\rangle\right|\nonumber\\
&\le Q(t)+\sqrt{tU_T}.
\label{eq:queuetoccv}
\end{align}
This proves \eqref{eq:mainccv}. The argument is deterministic on the same oracle event for every prefix and every feasible benchmark.

For the random realized quantities, the differences 
\[
\langle f^\star(x_\tau),\pi_\tau(\cdot\mid x_\tau)\rangle-f_\tau(x_\tau,a_\tau),
\qquad
g_\tau(x_\tau,a_\tau)-\langle g^\star(x_\tau),\pi_\tau(\cdot\mid x_\tau)\rangle, \quad \tau \ge 1,
\]
have conditional mean zero and absolute value at most two. Thus their partial sums are martingales. We now apply Azuma--Hoeffding inequality to each partial sum at every $t$ with failure probability $\delta/(3T)$, and take a union bound over $t$. Each of the two resulting events has probability at least $1-\delta/3$. On these events, \eqref{eq:realizedregret} and \eqref{eq:realizedccv} follow from their definitions. Intersecting with the oracle event \eqref{eq:oracle} proves the theorem. 
\end{proof}

\section{Conclusion}
We proposed an algorithm for constrained contextual bandits with adversarial contexts and a benchmark feasible in expectation at every round. The algorithm uses online regression to estimate rewards and costs, applies the positive part to the expected cost of an action distribution, and accumulates its predicted value. A single regret decomposition inequality and an exponential Lyapunov function give regret of order $O(\sqrt{KTU_T})$ in conditional means and cumulative constraint violation of order $\widetilde O(\sqrt{KTU_T}\log T)$, without any margin condition.
%
\section{Acknowledgment}
This work is supported by the Department of Atomic Energy, Government of
India, under Project No.~RTI4014 and a Google India faculty research award.

\section*{AI Disclosure}
The authors interacted with GPT 6 Astra to assist with drafting and revising the text, and checking and simplifying mathematical arguments.
\newpage

\bibliography{references, OCO.bib}
\end{document}